\documentclass{article}

\usepackage{microtype}
\usepackage{stfloats}
\usepackage{placeins}
\usepackage{graphicx}
\graphicspath{{../figures/}}
\usepackage{booktabs}
\usepackage{float}

\usepackage{xcolor}
\usepackage{hyperref}

\usepackage{latexml}
\iflatexml\usepackage{authblk}\fi

\usepackage[preprint]{icml2026}
\hypersetup{
  pdftitle={SFT Overtraining Predicts Rank Inversion via Entropy Collapse Under RLVR},
  pdfauthor={Siddharth Aphale, Kelly Liu},
  pdfsubject={Accepted at the Deep Learning for Code (DL4C) Workshop at ICML 2026},
  pdfkeywords={reinforcement learning, supervised fine tuning, GRPO, entropy collapse, code generation, rank inversion},
}

\usepackage{amsmath}
\usepackage{amssymb}
\usepackage{amsthm}
\usepackage{enumitem}

\newtheorem{proposition}{Proposition}
\newtheorem{corollary}{Corollary}
\newtheorem{remark}{Remark}
\newtheorem{definition}{Definition}
\newtheorem{assumption}{Assumption}

\let\oldthebibliography\thebibliography
\renewcommand{\thebibliography}[1]{%
  \oldthebibliography{#1}%
  \hbadness=3000}

\icmltitlerunning{SFT Overtraining Predicts Rank Inversion via Entropy Collapse Under RLVR}

\begin{document}

\iflatexml
\title{SFT Overtraining Predicts Rank Inversion via Entropy Collapse Under RLVR}
\author[1]{Siddharth Aphale}
\author[1]{Kelly Liu}
\affil[1]{Stanford University, Stanford, CA, USA}
\maketitle
\begin{center}
Correspondence to: Siddharth Aphale {\tt <saphale@stanford.edu>}.\\
\textit{Accepted at the Deep Learning for Code (DL4C) Workshop at ICML 2026.}
\end{center}
\else
\twocolumn[
\icmltitle{SFT Overtraining Predicts Rank Inversion via Entropy Collapse Under RLVR}

\begin{icmlauthorlist}
\icmlauthor{Siddharth Aphale}{stanford}
\icmlauthor{Kelly Liu}{stanford}
\end{icmlauthorlist}

\icmlaffiliation{stanford}{Stanford University, Stanford, CA, USA}

\icmlcorrespondingauthor{Siddharth Aphale}{saphale@stanford.edu}

\icmlkeywords{reinforcement learning, supervised fine tuning, GRPO,
              entropy collapse, code generation, rank inversion}

\vskip 0.3in
]

\printAffiliationsAndNotice{\textit{Accepted at the Deep Learning for Code
(DL4C) Workshop at ICML 2026.}}
\fi

\begin{abstract}
  The standard heuristic of selecting the SFT checkpoint with the
  highest pass@1 for GRPO can fail when SFT compresses the rollout
  distribution. For binary rewards, the expected within group
  advantage variance is $p(1{-}p)(g{-}1)/g$; when early GRPO drives
  $p$ below $p^*(g)$, most groups have identical rewards and provide
  no group relative signal. We study SFT depth ladders for
  Qwen2.5-Coder-3B and DeepSeek-Coder-6.7B. We test
  Qwen2.5-Coder-3B across five depths and three seeds, and
  DeepSeek-Coder-6.7B across four matched depths and three seeds.
  On Qwen, pre RL pass@1 rises with SFT depth, but peak GRPO
  pass@10 falls from $0.806$ to $0.481$ (3 seed mean, $n{=}20$);
  pre RL entropy is positively associated with the GRPO outcome
  ($\rho{=}{+}0.69$). On DeepSeek, pass@1 remains far above
  $p^*(8){=}0.083$, and GRPO outcomes compress rather than invert.
  A two stage diagnostic, combining pre RL entropy triage with an
  early GRPO entropy monitor, flags high risk checkpoints and can stop
  failing runs early. Simple KL to reference regularisation and label
  smoothing variants do not rescue the collapsed Qwen checkpoint in
  our setting, suggesting the failure is not a trivial GRPO
  hyperparameter artefact.
\end{abstract}

\vspace{-0.38in}
\begin{quote}
  \small\url{https://github.com/siddharthaphale/entropy-collapse-rlvr}
\end{quote}

\section{Introduction}

The standard post training recipe for code generation applies
supervised fine tuning (SFT) then reinforcement learning with
verifiable rewards (RLVR)~\citep{shao2024deepseekmath,
guo2024deepseekcoder}, selecting the highest scoring SFT
checkpoint for RL. This rule is increasingly questioned:
prolonged SFT memorises rather than
generalises~\citep{chu2025sft}, RLVR narrows the reasoning
boundary~\citep{yue2025rlreasoning}, and pass@1 alone is a
weak predictor of post RL outcome at
scale~\citep{kang2025quagmires}. We show this criterion is
misleading when SFT overtraining is associated with the policy
entering entropy collapse:
across an SFT depth ladder on Qwen2.5-Coder-3B-Base, peak
GRPO pass@10 falls monotonically from $0.806$ to $0.481$
(3 seed mean) while pre RL pass@1 \emph{rises}, and the
highest pass@1 checkpoint loses to shallower counterparts
in every seed. The mechanism: SFT overtraining compresses
output diversity, extinguishing the gradient signal GRPO
requires. A parallel ladder on DeepSeek-Coder-6.7B-Base
provides a regime boundary (\S\ref{sec:rank_inversion}).

Prior work on SFT to RL transitions has focused on data
composition~\citep{chu2025sft} or offline versus online distribution
mismatch~\citep{zhang2026pear} rather than checkpoint specific
selection at fixed compute. Aggregate predictors do not
diagnose individual failures. Pass@64 combined with
generalisation loss predicts post RL outcome at
$R^2{=}0.94$~\citep{kang2025quagmires}, far above pass@1
alone, but two checkpoints with identical pass@64 can carry
radically different entropy profiles. Pass@$k$ at large $k$
measures whether problems are ever solvable rather than how
reliably; the checkpoint specific variance question is by
construction outside its scope~\citep{dragoi2025beyond,
yue2025rlreasoning}. Greedy pass@1 ($T{=}0$) measures
capability at the wrong temperature; the rollout distribution
at $T{=}1.0$, where GRPO operates, may be far more compressed
than greedy behaviour suggests.

We make three bounded contributions. First, under binary rewards
the within group advantage variance is exactly
$\mathbb{E}[\sigma_G^2] = p(1{-}p)(g{-}1)/g$, yielding a
majority degenerate threshold $p^*(g)$ at which GRPO's
group relative signal structurally collapses
(Proposition~\ref{prop:threshold}). Second, on
Qwen2.5-Coder-3B-Base, we identify a rank inversion case study:
SFT depth raises pre RL pass@1 while lowering peak GRPO pass@10,
and the failure is explained by entropy collapse and reward variance
collapse. Third, on DeepSeek-Coder-6.7B-Base, pass@1 remains far
above $p^*(8)$ and GRPO ranks compress rather than invert, providing
a contrastive safe regime. The exact entropy cutoffs used for
triage are calibrated to the Qwen ladder; the model agnostic claim is
the checkpoint ordering and risk signal, not universal thresholds.
Two natural interventions, a KL penalty and label smoothing, fail to
rescue collapsed checkpoints in our setting; the failure appears
upstream of the tested GRPO variants.
We validate on three random seeds. Section~\ref{sec:related}
reviews related work; \S\ref{sec:methods} the experimental
setup; \S\ref{sec:results} the rank inversion, mechanism,
diagnostic, cross model validation, and interventions;
formal derivations in Appendix~A.3.

\section{Related Work}
\label{sec:related}

\paragraph{SFT overtraining constrains RL.}\hbadness=1300\relax
SFT memorises solutions while RL generalises across
them~\citep{chu2025sft}. In \citet{kang2025quagmires}'s cross model
study, pass@1 is a far weaker predictor of post RL outcome than
pass@64 with generalisation loss, and extending SFT beyond two epochs
degrades GRPO outcomes across hundreds of models. In our within family
ladder pass@1 correlates \emph{negatively} with post RL peak pass@10
($\rho{=}{-}0.75$, \S\ref{sec:rank_inversion}).  \citet{zhang2026pear} independently demonstrate rank inversion in mathematical reasoning,
attributing it to an offline versus online distribution mismatch and proposing PEAR as a remedy;
our entropy analysis provides the mechanistic account via Proposition~\ref{prop:variance}
and identifies the failure at the SFT stage rather than the SFT loss objective.

\paragraph{Limitations of pass@$k$ as a readiness proxy.}
Large-$k$ pass@$k$ measures whether a problem is \emph{ever} solvable
rather than \emph{how reliably}~\citep{dragoi2025beyond}. Because RLVR
narrows the reasoning boundary rather than expanding it, pass@$k$
captures the capability ceiling but not whether RL compression will
succeed~\citep{yue2025rlreasoning}. Proposed alternatives that target
output spread directly, like the diversity ratio $\Delta_k$, fall short
at low pass@1 (\S\ref{sec:entropy_collapse}).
Greedy pass@1 ($T{=}0$) fails for a distinct but related reason: it
measures capability at the wrong temperature. GRPO operates at $T{=}1.0$,
where a checkpoint's output distribution may be far more compressed than
its greedy behaviour suggests.

\paragraph{Entropy as the binding constraint.}
\citet{cui2025entropy} show that the RL performance ceiling is
determined by a model's initial entropy, and propose Clip-Cov to
address collapse \emph{during} RL. Our diagnostic addresses the prior
stage: whether SFT has already depleted entropy before GRPO begins.
Our pre RL screen complements their Clip-Cov method: we identify which
checkpoints need entropy preservation before training, they sustain it
during training.

\paragraph{RL algorithms and the diversity requirement.}
GRPO's critic free design makes the failure mode maximally legible:
when all rollouts in a group receive identical rewards, the
group relative advantage is zero and the gradient
vanishes~\citep{shao2024deepseekmath}. PPO and RLOO partially mask
this through their baselines; we use GRPO precisely because it makes
entropy collapse observable in the training signal. Training time
entropy preservation methods~\citep{chen2025passk,
walder2025passk} are complementary; our diagnostic identifies which
checkpoints need them before training begins.

\paragraph{Entropy collapse as a symptom of plasticity loss.}
Loss of plasticity (the progressive inability of a network to adapt
to new training signals after extended optimisation) is well
documented in deep RL~\citep{lyle2023plasticity, nikishin2022primacy,
dohare2024plasticity}. Entropy collapse is consistent with this
picture, with the practical advantage that entropy is measurable
pre RL whereas effective rank and dead neuron fraction require
additional instrumentation. Plasticity preserving methods, including periodic
network resets~\citep{nikishin2022primacy} and regularisation toward
initial weights~\citep{kumar2023regenerative}, are complementary:
our diagnostic identifies which checkpoints have already undergone
entropy collapse before RL compute is spent, while those methods
address sustaining plasticity during training.

\section{Methods}
\label{sec:methods}

We isolate SFT duration as the sole independent variable in our
main study: all checkpoints share identical architecture, data,
hyperparameters, and evaluation protocol. Ablations in
\S\ref{sec:mitigation} vary the GRPO KL penalty and apply label
smoothing to the 5.8 epoch checkpoint. The cross model thread on
DeepSeek-Coder-6.7B-Base appears throughout \S\ref{sec:results}
alongside the Qwen ladder and is summarised in Appendix~A.5.

\paragraph{Models.} Two base models are trained under identical
recipes. The primary ladder uses
Qwen2.5-Coder-3B-Base~\citep{qwen2_5_coder}; the cross model
ladder uses DeepSeek-Coder-6.7B-Base~\citep{guo2024deepseekcoder}.
Both are fine tuned in BF16 with LoRA ($r{=}128$, $\alpha{=}128$,
applied to all linear layers plus embeddings). SFT uses AdamW
8 bit, learning rate $1{\times}10^{-5}$, batch size 16, constant
schedule, weight decay $0.001$. A 100 step format warmup teaches
the \texttt{<think>}$\to$code output format before
checkpoint specific training begins.

\paragraph{SFT dataset and checkpoints.} Training uses 5,000
medium difficulty KodCode-V1 problems~\citep{kodcode2025},
decontaminated against HumanEval and MBPP via embedding cosine
similarity (threshold $\tau{=}0.70$, all-MiniLM-L6-v2, dropping $20.8\%$
of the pool). CoT traces were regenerated with Gemini 2.5 Flash targeting
concise reasoning (median 518 tokens) to fit the 2,048 token context
window without truncation. Five checkpoints spanning 1.0 to 5.8 epochs are selected for GRPO on
Qwen (Table~\ref{tab:sft_selection}); we refer to checkpoints by
their SFT epoch value throughout the paper. On
DeepSeek-Coder-6.7B-Base the same recipe yields eight SFT
checkpoints (1.0 to 9.6 epochs). Four matched checkpoints
(1.0, 1.9, 3.8, and 5.8 epochs) are GRPO-trained on three
seeds.\footnote{Seeds 42, 123, 456.}

\paragraph{GRPO dataset.} We construct the GRPO training set from
\texttt{KodCode-V1}~\citep{kodcode2025} problems not used during SFT
training, with an identical recipe across both models. Candidates are
decontaminated in three passes against (i) the SFT training set
($\tau{=}0.75$), (ii) HumanEval+ and MBPP ($\tau{=}0.70$), and (iii)
the frozen 40 problem deep eval subset ($\tau{=}0.70$), all via
all-MiniLM-L6-v2 cosine similarity. For each surviving candidate we
run 16 stochastic rollouts ($T{=}1.0$, top $p{=}0.95$) from a
calibration checkpoint at the middle of the GRPO ladder, count
passing rollouts via pytest, and retain problems with
pass\_count $\in [1, 14]$, excluding both unsolvable problems
(pass\_count${=}0$) and saturated ones (pass\_count${\geq}15$). This
``calibration band'' ensures the calibration model has signal on
every kept problem but is not saturated, so policy gradients carry
variance for all GRPO arms. The calibration checkpoint is the 2.9 epoch SFT for Qwen
($\to 1{,}096$ records) and the 3.8 epoch SFT for DeepSeek
($\to 1{,}104$ records).

\paragraph{GRPO training.} Each checkpoint undergoes identical 400 step
GRPO~\citep{shao2024deepseekmath} with the DAPO
variant~\citep{dapo2025} (loss in Appendix~A.2):
group size $g{=}8$, $\beta{=}0$ (no KL penalty),
$\varepsilon_{\mathrm{high}}{=}0.28$, learning rate
$1{\times}10^{-6}$, gradient clip $0.1$. The same LoRA adapter from SFT
is continued, eliminating adapter capacity as a confound.

\paragraph{Reward and evaluation.} We use binary correctness reward
($+2.0$ if all unit tests pass, $0$ otherwise). Format rewards are
excluded to prevent masking of the capability signal. Evaluation runs every
50 GRPO steps on a frozen 40 problem HumanEval+
subset~\citep{evalplus2023} ($n{=}20$, $T{=}1.0$), reporting
pass@$\{1,10\}$ via the unbiased estimator
of~\citet{humaneval_2021} (Appendix~A.1).
The 40 problem IDs are fixed across runs and preserve the
HumanEval+ difficulty distribution.
Entropy is probed every 10 steps on a 5 problem subset as mean
next token entropy $-\sum_v p_v \log p_v$ in nats, measured at the
generation start point via a single forward pass on prompt tokens only
(no completions generated).
The in training probe takes the per step minimum across these 5
health check problems while the pre RL probe (Table~\ref{tab:sft_selection})
averages over 40 problems; the orderings agree.

\section{Results}
\label{sec:results}

Selecting the highest scoring SFT checkpoint for GRPO, the standard
practice, picks the worst GRPO initialiser in every seed at every
depth, by a margin of $0.325$ in peak pass@10
($0.806$ at 1.0 epochs vs.\ $0.481$ at 5.8 epochs;
Table~\ref{tab:sft_selection}).
This failure is mechanical: SFT overtraining compresses the output
distribution and leaves deeper checkpoints at high risk of crossing the
gradient vanishing threshold $p^*(8)$ during early GRPO.
\S\ref{sec:rank_inversion} establishes the
rank inversion, \S\ref{sec:entropy_collapse} develops entropy collapse
as its mechanism, \S\ref{sec:diagnostic} operationalises this as a
two stage diagnostic, and \S\ref{sec:mitigation} tests simple
post hoc GRPO and SFT interventions.

\begin{table*}[t]
\centering
\caption{SFT ladder summary. Pre RL entropy and pass@1
($T{=}1.0$, $n{=}128$); GRPO peak pass@10 = max over canonical
50 step intervals of the 3 seed mean ($n{=}20$, in training).
Mean $\pm$ half range across 3 seeds. Absolute pass@10 levels
are not comparable across models (different family, capability
ceiling, completion length); within ladder ordering is the
comparison of interest. Per arm tables: Appendix~A.5.}
\label{tab:sft_selection}
\small
\begin{tabular}{rccc}
\toprule
SFT epochs & Entropy (nats) & Pre RL pass@1 ($n{=}128$) & GRPO Peak pass@10 ($n{=}20$) \\
\midrule
\multicolumn{4}{l}{\textit{Qwen2.5-Coder-3B-Base}} \\
1.0 & $0.227 \pm 0.045$ & $0.151 \pm 0.012$ & $0.806 \pm 0.028$ \\
1.9 & $0.163 \pm 0.008$ & $0.163 \pm 0.008$ & $0.750 \pm 0.079$ \\
2.9 & $0.156 \pm 0.026$ & $0.159 \pm 0.008$ & $0.671 \pm 0.031$ \\
3.8 & $0.171 \pm 0.029$ & $0.173 \pm 0.008$ & $0.608 \pm 0.026$ \\
5.8 & $0.120 \pm 0.019$ & $0.187 \pm 0.004$ & $0.481 \pm 0.044$ \\
\midrule
\multicolumn{4}{l}{\textit{DeepSeek-Coder-6.7B-Base}} \\
1.0 & $0.399 \pm 0.043$ & $0.351 \pm 0.019$ & $0.861 \pm 0.015$ \\
1.9 & $0.274 \pm 0.012$ & $0.383 \pm 0.005$ & $0.891 \pm 0.030$ \\
3.8 & $0.285 \pm 0.030$ & $0.413 \pm 0.013$ & $0.881 \pm 0.026$ \\
5.8 & $0.185 \pm 0.022$ & $0.413 \pm 0.009$ & $0.888 \pm 0.024$ \\
\bottomrule
\end{tabular}
\end{table*}

\begin{figure}[H]
  \centering
  \includegraphics[width=\columnwidth]{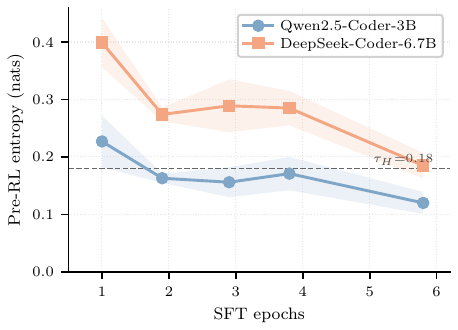}
  \caption{\textbf{Pre RL entropy across the SFT ladder.} Mean
    next token entropy on a 40 problem HumanEval+ probe ($T{=}1.0$,
    $n{=}128$); lines = 3 seed mean, bands = per checkpoint min/max.
    Dashed: Stage 1 threshold $\tau_H{=}0.18$~nats.}
  \label{fig:sft_curve}
\end{figure}

\subsection{Rank Inversion and Rank Compression Under GRPO}
\label{sec:rank_inversion}

The inversion is present from the first evaluation at step~50 and
holds without exception. The 1.0 epoch policy peaks at pass@10 of
$0.824$ at step~250, while the 5.8 epoch policy peaks at $0.527$
at step~50 and never recovers, a $2.37\times$ final gap. Notably,
the peak step itself contracts monotonically with SFT depth, from
step~250 to step~50, indicating that checkpoints with collapsed entropy
exhaust their learning signal earlier rather than learning more
efficiently. Pass@1 shows the same inversion: the 1.0 epoch policy
sustains $0.276$ to $0.335$ throughout training while the 5.8 epoch
policy collapses to $0.020$ by step~200. The full checkpoint specific
values appear in Table~\ref{tab:sft_selection}; the pattern holds
across all three seeds, with pre RL entropy and pass@1 predicting
peak pass@10 at Spearman $\rho{=}{+}0.69$ and $-0.75$ respectively
($p{<}0.01$, seed demeaned; Appendix~A.6). The implication for
practice is direct: the standard rule of selecting the highest
post SFT pass@1 checkpoint for RL systematically identifies the
worst GRPO initialiser in the entropy collapse regime.

The same protocol on DeepSeek-Coder-6.7B-Base produces the
opposite regime, rank \emph{compression}: post RL pass@10 bunches
in $[0.841, 0.884]$ (3 seed mean, $n{=}128$) and pre RL to
post RL ranks coincide exactly (Spearman $\rho{=}{+}1.00$,
vs.\ $\rho{=}{-}0.75$ on Qwen, $p{<}0.01$, $n{=}15$;
Figure~\ref{fig:rank_inversion}b, Appendix~A.6). Every DeepSeek arm
improves under GRPO ($\Delta$pass@10 $\in [+0.003, +0.024]$,
Table~\ref{tab:deepeval_crossmodel}). DeepSeek thus serves as a
contrastive validation of the threshold prediction, not a
replication of the collapse mechanism: it occupies the safe side of
$p^*(8)$ where the predicted failure does not occur.
\S\ref{sec:entropy_collapse} shows why the entropy collapse regime
is the boundary between the two outcomes.

\begin{figure*}[t]
  \centering
  \includegraphics[width=\textwidth]{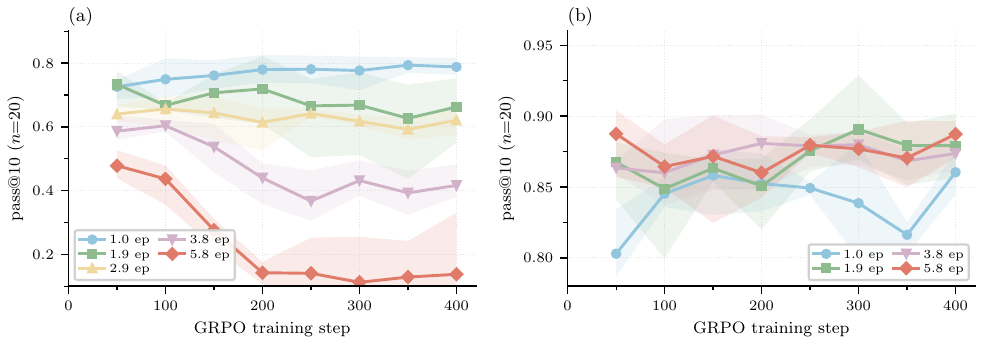}
  \caption{\textbf{Rank inversion vs.\ rank compression}
    ($n{=}20$, $T{=}1.0$, in training pass@10). Lines: 3 seed
    mean; bands: min/max. (a)~Qwen2.5-Coder-3B-Base: curves fan
    out with SFT depth, post RL ordering inverts pre RL.
    (b)~DeepSeek-Coder-6.7B-Base: four matched three seed depths
    bunch in $[0.80,0.91]$. Pre and post RL ranks coincide exactly
    (Spearman $\rho{=}{+}1.00$). Y axes differ between panels;
    absolute levels not comparable across models.}
  \label{fig:rank_inversion}
\end{figure*}

\subsection{Entropy Collapse as Mechanism}
\label{sec:entropy_collapse}

At the SFT level, pre RL entropy falls from $0.227$ nats at 1.0
epochs to $0.120$ at 5.8 epochs (3 seed mean,
Table~\ref{tab:sft_selection}), with the seed 42 nonmonotonicity
at 3.8 epochs ($0.192$ nats) discussed below
(Figure~\ref{fig:entropy_collapse}).
Yet pre RL pass@64 varies within a narrow 3.5 point band
($0.897$ to $0.932$, $n{=}128$, 3 seed mean), indicating a similar
capability ceiling across the ladder: the checkpoints differ not in
\emph{whether} they can solve problems but in \emph{how reliably}
they sample solutions~\citep{kang2025quagmires}.
During GRPO, worst token entropy collapses across all checkpoints,
but the severity scales with SFT depth. The shallowest checkpoint
(1.0 epochs) loses $37\%$ of its step 10 entropy by step~400
($0.078 \to 0.049$~nats); the deepest (5.8 epochs) loses $78\%$
($0.018 \to 0.004$~nats), ending the run at $12\times$ lower
entropy than the shallowest (Figure~\ref{fig:entropy_collapse}a).
The nonmonotonicity at 3.8 epochs (seed=42) is discussed below.

\begin{figure*}[t]
  \centering
  \includegraphics[width=\textwidth]{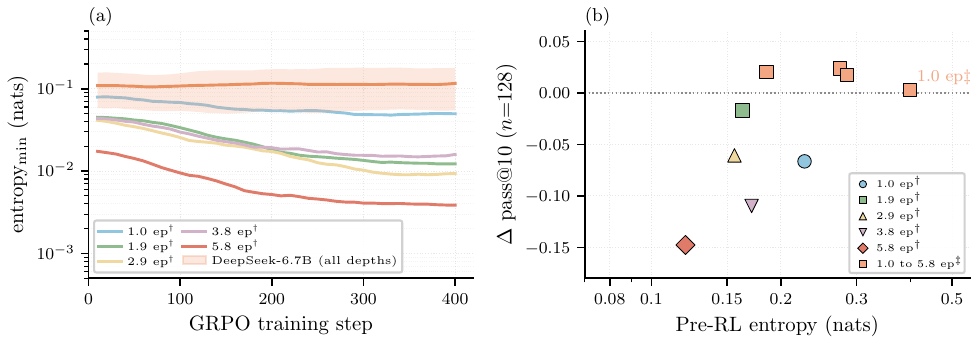}
  \caption{\textbf{Entropy collapse and GRPO outcomes.}
    (a)~Per step worst token entropy ($n{=}5$ probe, log scale,
    3 seed mean): Qwen$^{\dagger}$ collapses with SFT depth;
    DeepSeek$^{\ddagger}$ band is the per step min/max across
    four GRPO checkpoints. (b)~Pre RL entropy vs.\ $\Delta$pass@10
    (post RL $-$ pre RL deep eval, $n{=}128$, 3 seed mean): Qwen
    inverts ($\Delta{<}0$), DeepSeek compresses ($\Delta{\geq}0$).
    $^{\dagger}$Qwen-3B;
    $^{\ddagger}$DeepSeek-6.7B.}
  \label{fig:entropy_collapse}
\end{figure*}

\paragraph{Pass@10 trajectory (seed=42, $n{=}20$).}
For the 5.8 epoch policy, pass@10 drops sharply between steps~50 and
200 (from $0.527$ to $0.176$, a $67\%$ decline), with pass@1 falling
from $0.085$ to $0.020$ over the same window. After step~200 the policy
partially recovers to pass@10 of $0.331$ by step~400 but never regains
its step 50 value. By contrast, the 1.0 epoch policy improves from
$0.748$ at step~50 to a peak of $0.824$ at step~250, then stabilises
within $5\%$ of its peak through step~400.  The entropy collapse
mechanism predicts this pattern and the per step entropy trajectories
confirm it: the 5.8 epoch worst token entropy drops from $0.013$ nats
at step~50 to $0.005$ at step~200 (a $62\%$ decline), the same window
in which pass@10 collapses.  The 1.0 epoch entropy stays above
$0.040$ nats throughout (Figure~\ref{fig:entropy_collapse}a).

\paragraph{Interpretation.}
Rather than equalising the pre RL entropy gap, GRPO amplifies it:
the worst token entropy ratio between the 1.0 epoch and 5.8 epoch
checkpoints grows from $4.3\times$ at step~10 to $12.3\times$ by
step~400 (3 seed mean). The diversity ratio $\Delta_k$ nevertheless
stays above $0.84$ throughout, confirming that the collapse is in
\emph{sampling probability} (how often the model generates a correct
solution) rather than in the upper bound on solvable problems.

\paragraph{Mechanistic account: from low entropy to gradient death.}
The data support a self reinforcing collapse cycle with five identifiable
stages:
\begin{enumerate}[leftmargin=*,itemsep=1pt,parsep=0pt]
  \item \emph{Low SFT entropy:} Extended SFT compresses the output
    distribution: $0.227$ nats at 1.0 epochs falling to $0.120$ at
    5.8 epochs, 3 seed mean on the 40 problem probe.
  \item \emph{Homogeneous rollouts:} At GRPO temperature $T{=}1.0$, the
    low entropy policy generates near identical completions across the
    group of $g{=}8$, collapsing pass@1 toward~$0$ (e.g.\ $0.020$ for
    the 5.8 epoch policy at step~200, down from $0.085$ at step~50).
  \item \emph{Reward variance collapse:} By
    Proposition~\ref{prop:variance},
    $\mathbb{E}[\sigma_G^2] = p(1{-}p)(g{-}1)/g$. At $p{=}0.020$ this
    is $0.0172$, with most groups degenerate
    (Proposition~\ref{prop:degenerate}), carrying zero gradient.
    The exact closed form of Proposition~\ref{prop:variance} is used
    directly throughout; Remark~\ref{rem:iid} in Appendix~A.3
    quantifies the i.i.d.\ approximation underlying it for
    theoretical completeness.
  \item \emph{Gradient death:} With near zero advantage variance, the
    GRPO update is negligible; the policy cannot move toward correct
    solutions.
  \item \emph{Further collapse:} Without gradient signal the policy
    stagnates; residual optimiser momentum and weight decay further
    erode capability, closing the loop.
\end{enumerate}
This cycle is \emph{self reinforcing}: once entered, each stage amplifies
the next.

\paragraph{Seed level entropy nonmonotonicity at 3.8 epochs.}
Pre RL entropy is broadly consistent with the rank inversion direction
but shows seed specific nonmonotonicity: in seed 42, the 3.8 epoch
checkpoint reads $0.192$~nats, higher than both 1.9 epochs ($0.163$)
and 2.9 epochs ($0.137$), yet 3.8 epochs still underperforms both on
peak pass@10 (seed=42: $0.585$ vs.\ $0.765$ and $0.636$).  By 3 seed
mean the ordering is monotone in SFT depth ($0.608$ at 3.8 epochs,
between $0.671$ at 2.9 and $0.481$ at 5.8), so the nonmonotonicity is
a seed level entropy artefact rather than an outcome anomaly.  It motivates
augmenting pre RL entropy triage with an early GRPO entropy monitor
(Stage 2), which tracks collapse rate and does not rely on a single
pre RL snapshot.

\paragraph{Contrastive validation of Proposition~\ref{prop:threshold}.}
Rank compression on DeepSeek (\S\ref{sec:rank_inversion},
Figure~\ref{fig:rank_inversion}b) is what
Proposition~\ref{prop:threshold} predicts: every DeepSeek pre RL
pass@1 ($0.351$ to $0.413$, 3 seed mean) is $\geq 4.2\times$ above
$p^*(8){=}0.083$, so no majority degenerate collapse is expected.
Under identical GRPO, DeepSeek entropy in fact stays above the
gradient death regime ($\mathbb{E}[\sigma_G^2]$ in $0.20$ to $0.21$
throughout, roughly $12\times$ the Qwen 5.8 epoch step 200 value); the
five stage cycle above cannot start. Proposition~\ref{prop:threshold}
therefore separates the observed regimes in a bounded sense: Qwen
starts above $p^*(8)$ pre RL but the deepest checkpoint crosses below
it during early GRPO, while every DeepSeek rung stays
$\geq 4.2\times$ above threshold
(Table~\ref{tab:deepeval_crossmodel}).

\begin{table}[!t]
\centering
\caption{Cross model deep eval pass@10 ($n{=}128$, 3 seed mean).
``vs $p^*(8)$'' = pre RL pass@1 normalised by $p^*(8){=}0.083$.
Comparability caveats as in Table~\ref{tab:sft_selection}.}
\label{tab:deepeval_crossmodel}
\small
\setlength{\tabcolsep}{4pt}
\begin{tabular}{lcccc}
\toprule
Epochs & Pre p@10 & Post p@10 & $\Delta$ p@10 & vs $p^*(8)$ \\
\midrule
\multicolumn{5}{l}{\textit{Qwen2.5-Coder-3B-Base}} \\
1.0 & 0.663 & 0.597 & $-0.067$ & $1.8\times$ \\
1.9 & 0.660 & 0.643 & $-0.017$ & $2.0\times$ \\
2.9 & 0.655 & 0.594 & $-0.061$ & $1.9\times$ \\
3.8 & 0.675 & 0.566 & $-0.110$ & $2.1\times$ \\
5.8 & 0.700 & 0.553 & $-0.148$ & $2.3\times$ \\
\midrule
\multicolumn{5}{l}{\textit{DeepSeek-Coder-6.7B-Base}} \\
1.0 & 0.838 & 0.841 & $+0.003$ & $4.2\times$ \\
1.9 & 0.860 & 0.884 & $+0.024$ & $4.6\times$ \\
3.8 & 0.852 & 0.870 & $+0.017$ & $5.0\times$ \\
5.8 & 0.854 & 0.875 & $+0.021$ & $5.0\times$ \\
\bottomrule
\end{tabular}
\end{table}

\subsection{Two Stage Diagnostic}
\label{sec:diagnostic}

\begin{table}[t]
\centering
\caption{Two stage diagnostic (Qwen2.5-Coder-3B calibration, seed=42).
Stage~1: pre RL entropy (nats); Stage~2: worst token entropy drop,
GRPO step 10$\to$150. Peak pass@10 ($n{=}20$): 3 seed mean $\pm$
half range. DeepSeek cross check below.}
\label{tab:diagnostic}
\small
\begin{tabular}{rrrr}
\toprule
SFT epochs & Stage 1 $H$ & Stage 2 drop & Peak p@10 \\
\midrule
1.0 & 0.249 & 24\% & $0.806\pm0.028$ \\
1.9 & 0.163 & 44\% & $0.750\pm0.079$ \\
2.9 & 0.137 & 56\% & $0.671\pm0.031$ \\
3.8 & 0.192 & 28\% & $0.608\pm0.026$ \\
5.8 & 0.120 & 64\% & $0.481\pm0.044$ \\
\bottomrule
\end{tabular}
\end{table}

The theory motivates a two stage protocol calibrated on the Qwen
ladder; the model agnostic claim is the \emph{ordering}
(lower pre RL entropy $\Rightarrow$ worse GRPO), not the specific
cutoffs.

\paragraph{Stage 1: Pre RL entropy triage (no GRPO compute).}
Flag a checkpoint as high risk if its mean next token entropy at
rollout temperature $T{=}1.0$ falls below a threshold:
\begin{equation}
  H(\pi_{\mathrm{SFT}}) \;<\; \tau_H
  \label{eq:stage1}
\end{equation}
where $H(\pi_{\mathrm{SFT}})$ is the entropy probe defined in
\S\ref{sec:methods} and we set $\tau_H = 0.18$~nats from the
per seed distribution in Table~\ref{tab:diagnostic}
($0.120$ to $0.249$ nats). Across $n{=}15$ observations (5
checkpoints $\times$ 3 seeds), pre RL entropy predicts peak GRPO
pass@10 at Spearman $\rho{=}{+}0.69$ ($p{<}0.01$, seed demeaned;
Table~\ref{tab:correlations}); with $\tau_H{=}0.18$, three of the
five checkpoints flag (1.9, 2.9, 5.8 epochs), and 5.8 epochs is
highest risk in every seed. Pre RL pass@1 is a strong
\emph{inverse} predictor ($\rho{=}{-}0.75$): the standard
highest pass@1 rule selects the worst checkpoint. The
theoretically grounded plug in
$\hat{\mathbb{E}}[\sigma^2_G] = \hat p(1{-}\hat p)(g{-}1)/g$
(Proposition~\ref{prop:variance}) is the quantity to screen on;
we use entropy instead because pre RL pass@1 spans only
$0.151$ to $0.187$ ($n{=}128$), giving limited discriminating signal.
This triage signal is valid when entropy reflects concentration on
plausible correct continuations; Remark~\ref{rem:paradox}
(Appendix~A.4) delineates the condition under which this breaks
down.

\paragraph{Stage 2: Early GRPO entropy monitor (step 150).}
During GRPO, flag collapse if
\begin{equation}
  \frac{\Delta H(10 \to 150)}{H(10)} > \tau_2
\end{equation}
where $H(10)$ is the worst token entropy at GRPO step~10 and we
set $\tau_2 = 0.50$. The relative drop broadly escalates with SFT
depth (Table~\ref{tab:diagnostic}): by step~150, the 5.8 epoch
checkpoint has lost $2.7\times$ as much entropy as the 1.0 epoch
checkpoint ($64\%$ vs.\ $24\%$). Under $\tau_2{=}0.50$, the
2.9 epoch checkpoint is also conservatively flagged ($56\%$ drop),
while the 1.9 epoch checkpoint is Stage 1 flagged but not Stage 2
flagged. Stopping a flagged run at step~150 rather than step~400 saves
$62.5\%$ of RL compute; Stage~1 can be used as a cautionary triage
signal before committing to a full GRPO run.

\paragraph{From diagnostic to selection rule.}
Treating 5.8 epochs as the highest risk checkpoint and selecting the
1.9 epoch optimum from
\S\ref{sec:rank_inversion} (in place of the standard pass@1
winner, the 5.8 epoch checkpoint) recovers $+0.090$ absolute on
deep eval pass@10 ($0.643$ vs.\ $0.553$, $n{=}128$, 3 seed mean;
Table~\ref{tab:deepeval_crossmodel}). DeepSeek independently places
its SFT optimum at 1.9 epochs ($+0.009$ deep eval pass@10, $0.884$
vs.\ $0.875$ at 5.8 epochs): both ladders converge on the same
SFT stopping recommendation even though the failure modes differ.

\paragraph{Diagnostic on DeepSeek: negative at the mean level.}
Applied to DeepSeek with the same $\tau_H{=}0.18$, Stage~1 flags
zero of four checkpoints at the 3 seed mean level
(entropies $0.185$ to $0.399$~nats, all above threshold;
Table~\ref{tab:sft_selection_full}); all
four arms then improve deep eval pass@10
($\Delta \in [+0.003, +0.024]$, Table~\ref{tab:deepeval_crossmodel}),
matching the contrastive validation of the threshold prediction in
\S\ref{sec:entropy_collapse}.

\subsection{Localizing the Failure to the SFT Stage}
\label{sec:mitigation}

We test whether the rank inversion is an artefact of the GRPO
hyperparameters or of the SFT objective by perturbing each
separately on the Qwen rank inversion regime: a KL to reference
penalty during GRPO and label smoothing SFT to restore pre RL
entropy. The rank inversion persists under both, suggesting the
failure is not a trivial GRPO hyperparameter artefact and appears
upstream of the tested GRPO variants. DeepSeek requires no such modification
(\S\ref{sec:entropy_collapse}).

\begin{table*}[tb]
\centering
\caption{Post hoc interventions: KL to reference penalty
($\beta{=}0.01$) and label smoothing (LS) restoration
($\alpha{=}0.1$, 5.8 epoch only). Seed=42, 400 GRPO steps;
pre RL pass@1 is $n{=}128$ (seed=42 only; 3 seed means
$0.151$/$0.187$ for 1.0/5.8 epochs in
Table~\ref{tab:sft_selection}), peak p@10 is $n{=}20$.}
\label{tab:interventions}
\small
\begin{tabular}{lclrr}
\toprule
SFT epochs & $\beta$ & Restoration & Pre RL pass@1 & Peak pass@10 \\
\midrule
1.0 & 0.0  & n/a               & 0.137 & 0.824 \\
5.8 & 0.0  & n/a               & 0.187 & 0.527 \\
1.0 & 0.01 & n/a               & 0.137 & 0.684 \\
5.8 & 0.01 & n/a               & 0.187 & 0.461 \\
5.8 & 0.0  & LS $\alpha{=}0.1$ & 0.600 & 0.214 \\
\bottomrule
\end{tabular}
\end{table*}

\paragraph{KL to reference penalty ($\beta{=}0.01$).}
Adding a positive KL to reference penalty regresses both checkpoints
below their $\beta{=}0$ baselines
(Table~\ref{tab:interventions}, Figure~\ref{fig:posthoc_interventions}a): peak
pass@10 falls from $0.824$ to $0.684$ at 1.0 epochs and from $0.527$
to $0.461$ at 5.8 epochs, preserving the rank inversion. We did not
tune $\beta$; this rules out a standard KL penalty at this value as
a rescue, not that no KL strength could recover trainability.

\begin{figure*}[t]
  \centering
  \includegraphics[width=\textwidth]{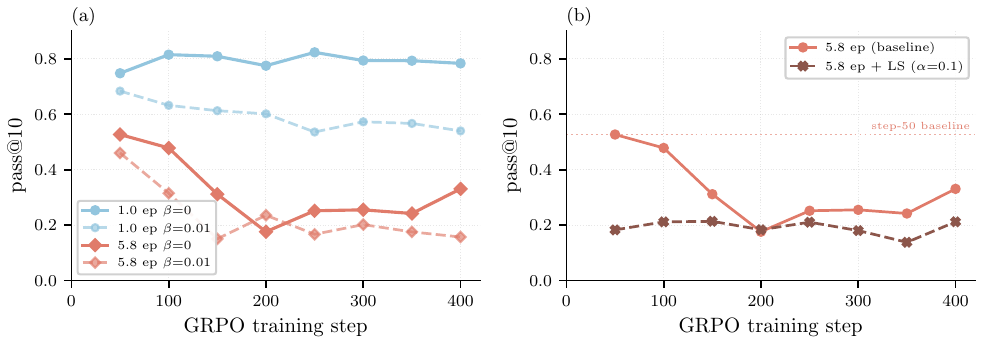}
  \caption{\textbf{Post hoc interventions fail to rescue trainability}
    (Qwen2.5-Coder-3B-Base). (a)~\emph{KL penalty preserves the rank
    inversion}: pass@10 under $\beta{=}0$ (solid) vs.\ $\beta{=}0.01$
    (dashed); both KL curves underperform their baselines.
    (b)~\emph{Entropy restoration paradox}: 5.8 epoch baseline vs.\
    5.8 epoch $+$ label smoothing ($\alpha{=}0.1$). LS yields the
    highest pre RL pass@1 of any run ($0.600$) but pass@10 caps at
    $0.214$.}
  \label{fig:posthoc_interventions}
\end{figure*}

\paragraph{Entropy restoration via label smoothing.}
Label smoothing on the 5.8 epoch checkpoint yields the
\emph{highest} pre RL pass@1 of any run in the study ($0.600$) yet
the \emph{lowest} peak GRPO pass@10 ($0.214$;
Table~\ref{tab:interventions}, Figure~\ref{fig:posthoc_interventions}b).
LS uniformly redistributes probability mass across the vocabulary,
preserving argmax ordering while flattening the sampling
distribution GRPO needs for advantage variance. This
\emph{restoration paradox} shows that entropy is necessary but not
sufficient for GRPO trainability.

Appendix~A.4 (Remark~\ref{rem:paradox}) formalises this: LS raises
entropy independently of whether the added mass falls on
correct solution tokens, violating the unimodal monotonicity
condition under which entropy is a valid diagnostic.
PEAR~\citep{zhang2026pear} (SFT loss reweighting before collapse)
and Clip-Cov entropy regularisation~\citep{cui2025entropy}
(GRPO side diversity control) may avoid these failure modes; we
leave empirical evaluation to future work.

\section{Conclusion}

SFT overtraining can precede entropy collapse under GRPO: the
resulting homogeneous rollouts are associated with pass@1 falling below
$p^*(8)$ during training, where the group relative advantage signal
vanishes (Proposition~\ref{prop:variance}). Pre RL entropy identifies
checkpoints at high risk of entering this regime, while $p^*(g)$
characterises the point at which GRPO's group relative signal
structurally collapses. In Qwen, the deepest checkpoint starts above
$p^*(8)$ pre RL but crosses below it during early GRPO; in DeepSeek,
pass@1 stays far above $p^*(8)$, and rank compression rather than
inversion is observed.

The consequence is concrete: pre RL pass@1, the standard
checkpoint selection criterion in RLVR pipelines, can
actively select the worst initialiser in the entropy collapse
regime, and replacing it with pre RL entropy triage
recovers $+0.090$ absolute on deep eval pass@10 ($0.643$
vs.\ $0.553$, $n{=}128$, 3 seed mean) at the cost of a single
forward pass.

\paragraph{Limitations and future work.}
Proposition~\ref{prop:variance} assumes binary rewards, so process
rewards, partial credit, and math reasoning remain natural stress tests.
Our evidence covers two code models; whether the same entropy and
advantage variance diagnostics calibrate across scale, model family, and
RL algorithms beyond GRPO/DAPO remains open. We also leave PEAR,
Clip-Cov, and other entropy preserving interventions to future work; a
successful rescue of collapsed checkpoints would provide stronger causal
evidence for entropy collapse as the driver of rank inversion.

\section*{Impact Statement}

The standard recipe of selecting the highest scoring post SFT checkpoint
for RL can select the worst trainable one.  Our pre RL entropy triage
requires no additional RL compute and can reduce wasted training runs.
We see no negative societal consequences specific to this work beyond
those common to language model research generally.

\FloatBarrier
\bibliography{references}
\bibliographystyle{icml2026}

\newpage
\onecolumn
\section*{Appendix}
\label{sec:appendix}

\subsection*{A.1. Low Variance pass@$k$ Estimation}

Directly computing pass@$k$ using only $k$ sampled outputs per problem
leads to high variance. We follow the unbiased estimation method
of~\citet{humaneval_2021}: for each problem $x_i$ from evaluation
dataset $\mathcal{D}$, we generate $n \geq k$ samples and count correct
samples $c_i$. The unbiased estimator is:
\[
  \mathrm{pass@}k
  := \mathbb{E}_{x_i \sim \mathcal{D}}
  \left[
    1 - \frac{\binom{n - c_i}{k}}{\binom{n}{k}}
  \right]
\]
For in training evaluations we set $n{=}20$; in deep evaluations we set $n{=}128$. The estimator supports any $k \leq n$.

\subsection*{A.2. GRPO Objective}
\label{sec:grpo_objective}

We use the DAPO variant of GRPO~\citep{shao2024deepseekmath, dapo2025}:
\begin{equation}
  \mathcal{L}_{\mathrm{GRPO}}(\theta)
  = -\,\mathbb{E}_{x,\{y_i\}}\!\left[
    \frac{1}{g}\sum_{i=1}^{g}
    \frac{1}{|y_i|}\sum_{t=1}^{|y_i|}
    \min\!\bigl(\rho_t^{(i)}\hat{A}_i,\;
    \mathrm{clip}\bigl(\rho_t^{(i)},\,1{-}\varepsilon,\,1{+}\varepsilon_{\mathrm{high}}\bigr)\hat{A}_i
    \bigr)\right]
\end{equation}
where $\rho_t^{(i)} = \pi_\theta(y_t^{(i)} \mid x, y_{<t}^{(i)}) /
\pi_{\theta_{\mathrm{old}}}(y_t^{(i)} \mid x, y_{<t}^{(i)})$ is the
per token importance ratio and $\hat{A}_i$ is the group relative
advantage.  Hyperparameters for this paper are listed in
Section~\ref{sec:methods}.

\subsection*{A.3. Formal Derivation of Gradient Vanishing}
\label{sec:gradient_vanishing}

\paragraph{Setup and Definitions.}
Consider a prompt $x$ and a policy $\pi_\theta$ generating responses
$y \sim \pi_\theta(\cdot \mid x)$ evaluated by a binary reward
$R(y) \in \{0,1\}$.  Let $p = P(R(y)=1)$ denote the per prompt success
probability.  GRPO generates a group $G = \{y_1,\ldots,y_g\}$
of $g$ independent responses, assigning group relative advantages
$A_i = r_i - \bar{r}$ where $\bar{r} = \frac{1}{g}\sum_j r_j$.
The within group advantage variance is
\[
  \sigma_G^2
  = \frac{1}{g}\sum_{i=1}^{g} A_i^2
  = \bar{r}(1-\bar{r})
\]
which equals zero (and the gradient vanishes) whenever all responses
yield the same reward.  The \emph{diversity ratio} is
\[
  \Delta_k
  = \frac{\mathrm{pass@}k - \mathrm{pass@}1}
         {\mathrm{pass@}k + \varepsilon}
\]
where $\varepsilon > 0$ is a stability constant.

\paragraph{Assumptions.}
The results below invoke four assumptions:
(A1) binary rewards $r_i \in \{0, 1\}$;
(A2) i.i.d.\ sampling of rollouts, $y_i \sim \pi_\theta(\cdot \mid x)$;
(A3) group size $g \geq 2$;
(A4) bounded score functions
$\|\nabla_\theta \log \pi_\theta(y \mid x)\|_2 \leq B$ for some $B < \infty$.

\begin{proposition}[Exact Advantage Variance]
\label{prop:variance}
Under i.i.d.\ sampling with binary rewards and group size $g \geq 2$:
\[
  \mathbb{E}[\sigma_G^2]
  \;=\;
  p(1-p) \cdot \frac{g-1}{g}
\]
where $p = \mathrm{pass@1}$ is the per problem success probability.
In particular, $\mathbb{E}[\sigma_G^2] = 0$ if and only if
$p \in \{0, 1\}$.
\end{proposition}

\begin{proof}
Let $S = \sum_i r_i \sim \mathrm{Binomial}(g, p)$.  Then
$\bar{r} = S/g$ and
\[
  \mathbb{E}[\sigma_G^2]
  = \mathbb{E}[\bar{r}(1-\bar{r})]
  = \mathbb{E}[\bar{r}] - \mathbb{E}[\bar{r}^2]
\]
Since $\mathbb{E}[\bar{r}] = p$ and
$\mathbb{E}[\bar{r}^2] = \mathrm{Var}(\bar{r}) + p^2
= p(1-p)/g + p^2$:
\begin{equation*}
  \mathbb{E}[\sigma_G^2]
  = p - \frac{p(1-p)}{g} - p^2
  = p(1-p) \cdot \frac{g-1}{g}
  \qedhere
\end{equation*}
\end{proof}

The result generalises to any $\{0,c\}$ reward scale as
$\mathbb{E}[\sigma_G^2] = c^2\,p(1{-}p)(g{-}1)/g$; the zero gradient
condition $p \in \{0,1\}$ is scale invariant, so our experimental
reward of $+2.0/0$ does not affect the qualitative conclusions.

\begin{corollary}[Gradient Signal Vanishing]
\label{cor:gradient}
Under assumptions (A1) through (A4) (binary rewards, i.i.d.\ sampling,
$g \geq 2$, bounded score functions
$\|\nabla_\theta \log \pi_\theta\|_2 \leq B$):
\[
  \mathbb{E}\bigl[\|\nabla_\theta J(x)\|_2^2\bigr]
  \;\leq\;
  B^2 \cdot p(1-p) \cdot \frac{g-1}{g}
\]
In particular, $\mathbb{E}[\|\nabla_\theta J\|_2] \to 0$
as $p \to 0$ or $p \to 1$.
\end{corollary}

\begin{proof}
Write $s_i = \nabla_\theta \log \pi_\theta(y_i \mid x)$ for the
score functions. By Cauchy Schwarz on finite sums:
\[
  \left\|\frac{1}{g}\sum_i A_i s_i\right\|^2
  \leq
  \left(\frac{1}{g}\sum_i A_i^2\right)
  \left(\frac{1}{g}\sum_i \|s_i\|^2\right)
\]
Therefore $\|\nabla_\theta J(x)\|^2 \leq \sigma_G^2 \cdot B^2$.
Taking expectations and applying Proposition~\ref{prop:variance}
yields the result.
\end{proof}

\begin{proposition}[Degenerate Group Probability]
\label{prop:degenerate}
The probability that a GRPO group is degenerate (all correct
or all incorrect, yielding zero gradient) is exactly:
\[
  P(\sigma_G^2 = 0)
  \;=\;
  (1-p)^g + p^g
\]
\end{proposition}

\begin{proof}
$\sigma_G^2 = 0$ iff $\bar{r} \in \{0,1\}$, i.e.\ $S \in \{0,g\}$.
The two events are mutually exclusive with probabilities $(1-p)^g$ and
$p^g$ respectively.
\end{proof}

\paragraph{Practical interpretation.}
For small $p$, the degenerate probability is dominated by
$(1-p)^g \approx 1 - gp$, meaning nearly all groups receive identical
(incorrect) rewards. For our worst collapsed checkpoint
the 5.8 epoch checkpoint at step~200 ($\mathrm{pass@1} = 0.020$) with group
size $g = 8$, this gives $P(\sigma_G^2 = 0) \approx 0.851$. In other
words, $85\%$ of GRPO groups carry zero gradient at the worst point of
training, leaving the policy effectively frozen.

We now sharpen Proposition~\ref{prop:degenerate} into a deterministic threshold on $p$ at which the majority of groups carry zero gradient by construction.

\paragraph{Collapse threshold.}\label{app:threshold}
We derive the \emph{majority degenerate threshold} $p^*(g)$, the pass@1 value below which
the majority of GRPO groups carry zero gradient by construction.

\begin{definition}[Majority degenerate regime]
A pass@1 value $p$ is \emph{majority degenerate} at group size $g$ if
$\Pr(\sigma_G^2 = 0) > \tfrac{1}{2}$, i.e.\ $f_g(p) := (1-p)^g + p^g > \tfrac{1}{2}$.
\end{definition}

\begin{proposition}[Existence and uniqueness of $p^*(g)$]\label{prop:threshold}
For every integer $g \ge 3$, there exists a unique $p^*(g) \in (0, \tfrac{1}{2})$ such that,
for $p \in [0, \tfrac{1}{2}]$:
\[
  f_g(p) > \tfrac{1}{2} \;\iff\; p < p^*(g), \qquad f_g\!\left(p^*(g)\right) = \tfrac{1}{2}.
\]
Moreover $f_g$ is strictly decreasing on $(0, \tfrac{1}{2})$.
\end{proposition}

\begin{proof}
We proceed in five steps.

\textbf{(i) Symmetry.} $f_g(1-p) = p^g + (1-p)^g = f_g(p)$, so it suffices to study
$f_g$ on $[0, \tfrac{1}{2}]$.

\textbf{(ii) Boundary values.} $f_g(0) = 1$ and $f_g(\tfrac{1}{2}) = 2 \cdot 2^{-g} = 2^{1-g}$.
For $g \ge 3$, $2^{1-g} \le \tfrac{1}{4} < \tfrac{1}{2}$.

\textbf{(iii) Continuity.} $f_g$ is a polynomial, hence continuous on $[0,1]$.

\textbf{(iv) Strict monotonicity.} Differentiating:
\[
  f_g'(p) = g\bigl[p^{g-1} - (1-p)^{g-1}\bigr].
\]
For $p \in (0, \tfrac{1}{2})$, $p < 1-p$, so $p^{g-1} < (1-p)^{g-1}$ and $f_g'(p) < 0$.
Hence $f_g$ is strictly decreasing on $(0, \tfrac{1}{2})$.

\textbf{(v) Existence and uniqueness via IVT.} Since $f_g(0) = 1 > \tfrac{1}{2} > 2^{1-g}
= f_g(\tfrac{1}{2})$, the Intermediate Value Theorem gives $p^*(g) \in (0, \tfrac{1}{2})$
with $f_g(p^*(g)) = \tfrac{1}{2}$. Strict monotonicity makes it unique. The biconditional
follows immediately. \qed
\end{proof}

\begin{remark}[i.i.d.\ approximation]\label{rem:iid}
Proposition~\ref{prop:threshold} assumes rollouts within a group are i.i.d.\ In practice,
completions share an autoregressive prefix up to the first diverging token; empirically this
occurs within 10 to 20 tokens for \textsc{HumanEval+} problems (mean completion length
$\approx$400 tokens), making the per rollout dependence negligible relative to the
group level variance. The threshold $p^*(g)$ is therefore robust to this approximation.
\end{remark}

\paragraph{Numerical values.}
We solve $(1-p)^g + p^g = \tfrac{1}{2}$ numerically for standard group sizes.
For $g \ge 8$, $p^g \le 0.083^8 \approx 2 \times 10^{-9}$ so the equation reduces to
$(1-p)^g \approx \tfrac{1}{2}$, giving the leading order approximation $p^*(g) \approx
1 - 2^{-1/g}$.

\begin{table}[h]
\centering
\caption{Majority degenerate threshold. $p^*(g)$ for standard GRPO group sizes.}
\label{tab:threshold}
\begin{tabular}{cccc}
\toprule
$g$ & $p^*(g)$ & $\mathbb{E}[\sigma_G^2]$ at $p^*(g)$ & Interpretation \\
\midrule
4  & 0.1594 & 0.1005 & $>$15\% pass@1 required \\
8  & 0.0830 & 0.0666 & $>$8.3\% pass@1 required \\
16 & 0.0424 & 0.0381 & $>$4.2\% pass@1 required \\
32 & 0.0214 & 0.0203 & $>$2.1\% pass@1 required \\
\bottomrule
\end{tabular}
\end{table}

\paragraph{Asymptotic behaviour.}
For large $g$, dropping $p^g = O(2^{-g})$ and solving $(1-p)^g = \tfrac{1}{2}$ gives
$p^*(g) \to 0$ at rate $\Theta(1/g)$ with leading term $p^*(g) \approx (\ln 2)/g$.
Doubling the rollout budget therefore only halves $p^*(g)$: a model in active collapse
($p \ll p^*(g)$) cannot recover trainability by increasing $g$ alone.

\paragraph{Empirical validation.}
We validate $p^*(8) = 0.083$ against the Qwen2.5-Coder-3B-Base 5.8 epoch checkpoint
at $g = 8$.

\textbf{Qwen2.5-Coder-3B-Base, 5.8 epochs.}
Pre RL pass@1 $p_0 = 0.187 > p^*(8)$: only $\approx 19\%$ of groups are degenerate at
initialization. During GRPO, pass@1 drops to $p_{200} = 0.020$ at step 200.
At this value:
\[
  \Pr(\sigma_G^2 = 0) = (0.98)^8 + (0.02)^8 \approx 0.8508,
\]
\[
  \mathbb{E}[\sigma_G^2] = 0.020 \times 0.980 \times \tfrac{7}{8} = 0.0172.
\]
Since $p_{200} = 0.020 < 0.083 = p^*(8)$, the model is squarely in the majority degenerate
regime: \textbf{85\% of groups produce zero gradient}, consistent with the 85.1\% reported
in Table~\ref{tab:diagnostic}.

$p^*(g)$ is a \emph{falsifiable prediction}: from a measured pass@1 value, it identifies
the regime in which gradient signal structurally vanishes. For the 5.8 epoch checkpoint,
pre RL pass@1 starts above threshold, and the threshold is crossed between GRPO steps 50
and 200 ($p$ falls from 0.085 to 0.020), after which no recovery is observed.

\paragraph{Bilateral empirical validation.}
Table~\ref{tab:deepseek_variance} instantiates
Proposition~\ref{prop:variance} for both ladders: the four
DeepSeek arms sit at $\mathbb{E}[\sigma^2_G] \in [0.20, 0.21]$
with $\Pr(\sigma^2_G{=}0) \leq 0.032$ throughout training, while
the Qwen 5.8 epoch checkpoint at GRPO step~200 collapses to
$\mathbb{E}[\sigma^2_G] = 0.017$ with $\Pr(\sigma^2_G{=}0) = 0.851$
($12.3\times$ smaller variance, $27\times$ higher degenerate group
probability), explaining the opposite GRPO outcomes.

\begin{table}[H]
\centering
\caption{Expected within group advantage variance from
Proposition~\ref{prop:variance} ($g{=}8$). DeepSeek values use
3 seed mean pre RL pass@1; the Qwen 5.8 epoch row is the
collapse point at GRPO step~200.}
\label{tab:deepseek_variance}
\small
\begin{tabular}{llrrr}
\toprule
Model & Epochs & $p$ & $\mathbb{E}[\sigma^2_G]$ & $\Pr(\sigma^2_G{=}0)$ \\
\midrule
DeepSeek-6.7B & 1.0 & 0.351 & 0.199 & 0.032 \\
DeepSeek-6.7B & 1.9 & 0.383 & 0.207 & 0.022 \\
DeepSeek-6.7B & 3.8 & 0.413 & 0.212 & 0.015 \\
DeepSeek-6.7B & 5.8 & 0.413 & 0.212 & 0.015 \\
\midrule
Qwen-3B & 5.8, step~200 & 0.020 & 0.017 & 0.851 \\
\bottomrule
\end{tabular}
\end{table}

\subsection*{A.4. Entropy as a Diagnostic Proxy}\label{app:entropy-proxy}

We characterise when pre RL entropy $H$ is a valid proxy for pass@1 $p$, and when it fails.

\paragraph{Unimodal distributions.}

\begin{assumption}[Unimodal correct mass]\label{asm:unimodal}
There is a single distinguished correct prefix event at the next token level with
probability $\pi(c) = p$, and the residual distribution $\tilde\pi$ on
$\mathcal{V} \setminus \{c\}$ has fixed shape: $\pi(v) = (1-p)\,\tilde\pi(v)$ for $v \ne c$.
This is a distributional analogy capturing the relevant monotonicity at the token level;
it is not a claim that a single vocabulary token determines correctness.
\end{assumption}

Under Assumption~\ref{asm:unimodal}, the chain rule for entropy gives:
\[
  H(\pi) = H_{\mathrm{bin}}(p) + (1-p)\,H(\tilde\pi),
\]
where $H_{\mathrm{bin}}(p) = -p \log p - (1-p)\log(1-p)$ is the binary entropy.
Differentiating with respect to $p$:
\[
  \frac{\mathrm{d}H(\pi)}{\mathrm{d}p} = \log\frac{1-p}{p} - H(\tilde\pi).
\]

\begin{proposition}[Unimodal monotonicity]\label{prop:monotonicity}
Under Assumption~\ref{asm:unimodal}, $H(\pi)$ is strictly increasing in $p$ on
$(0, \tfrac{1}{2})$ whenever $H(\tilde\pi) < \log[(1-p)/p]$.
\end{proposition}

\begin{proof}
The condition $\mathrm{d}H/\mathrm{d}p > 0$ is equivalent to
$H(\tilde\pi) < \log[(1-p)/p]$.
At $p = 0.02$ (our worst collapsed checkpoint), $\log(0.98/0.02) = \log 49 \approx 3.89$
nats, accommodating a residual distribution concentrated on up to $e^{3.89} \approx 49$
vocabulary tokens. This condition holds whenever the model's ``wrong'' mass is
concentrated on near correct alternatives, which is the typical regime during SFT.
\qed
\end{proof}

When Proposition~\ref{prop:monotonicity} holds, higher entropy $\Rightarrow$ higher pass@1
$\Rightarrow$ lower degenerate group probability. This makes entropy a valid early warning
signal for impending collapse under the diagnostic of Section~\ref{sec:diagnostic}.

\paragraph{Label smoothing breaks monotonicity.}

\begin{remark}[Restoration paradox]\label{rem:paradox}
Label smoothing with parameter $\alpha$ replaces the SFT cross entropy target with
$(1-\alpha)\delta_c + \alpha \cdot \mathrm{Uniform}(\mathcal{V})$, pushing the output
distribution toward a mixture of the base distribution and uniform. This raises $H(\pi)$
by injecting mass uniformly across $\mathcal{V}$, independently of whether that mass
falls on the correct prefix token $c$. When the base distribution is already peaked on $c$
(as at 5.8 epochs), smoothing specifically reduces $\pi(c) = p$ while increasing
$H(\pi)$: entropy and pass@1 move in opposite directions, violating the monotonicity
condition of Proposition~\ref{prop:monotonicity}. This explains the restoration paradox
empirically observed in Section~\ref{sec:mitigation}: 5.8 epochs + LS achieves the highest pre RL
entropy in our study yet the lowest peak GRPO pass@10, because the injected entropy
reflects vocabulary diffusion rather than concentration on correct continuations.
Entropy is a valid pre RL diagnostic only when it reflects concentration on a small set
of plausible continuations, not diffusion over the full vocabulary.
\end{remark}

\subsection*{A.5. Additional Empirical Results}
\label{sec:appendix_tables}
\setlength{\floatsep}{4pt}

Table~\ref{tab:sft_selection_full} gives the full cross model SFT
ladder; Figure~\ref{fig:peak_by_depth} visualises peak GRPO
pass@10 and the diversity ratio across both models. We extend the
DeepSeek ladder to 9.6 epochs (SFT only, no GRPO) to test whether its
resistance to rank inversion is merely a matter of SFT depth. Even at
${\approx}1.6\times$ the deepest matched checkpoint, pre RL entropy
plateaus ($0.19$ to $0.28$ nats, never approaching the Qwen 5.8 epoch
value of $0.120$) and pass@1 saturates ($0.41$ to $0.43$), so the
cross model difference is not an artefact of under training DeepSeek.

\begin{table*}[!t]
\centering
\caption{Cross model SFT ladder. Entropy on 40 problem HumanEval+
probe; pre RL pass@$\{1,10,64\}$ at $T{=}1.0$ ($n{=}128$); GRPO
peak pass@10 (in training, $n{=}20$). Mean $\pm$ half range
across seeds 42, 123, 456 where 3 seed data exist. The DeepSeek
2.9 epoch arm is seed=42 only; the 7.7/8.6/9.6 epoch rows are
seed=42 SFT diagnostic only (no GRPO).
$^{\dagger}$Seed=42 only.}
\label{tab:sft_selection_full}
\small
\setlength{\tabcolsep}{5pt}
\begin{tabular}{llccccc}
\toprule
Model & Epochs & Entropy & p@1 & p@10 & p@64 & GRPO peak p@10 \\
\midrule
Qwen-3B & 1.0 & $0.227\pm0.045$ & $0.151\pm0.012$ & $0.663\pm0.010$ & $0.932\pm0.011$ & $0.806\pm0.028$ \\
Qwen-3B & 1.9 & $0.163\pm0.008$ & $0.163\pm0.008$ & $0.660\pm0.016$ & $0.901\pm0.003$ & $0.750\pm0.079$ \\
Qwen-3B & 2.9 & $0.156\pm0.026$ & $0.159\pm0.008$ & $0.655\pm0.009$ & $0.907\pm0.011$ & $0.671\pm0.031$ \\
Qwen-3B & 3.8 & $0.171\pm0.029$ & $0.173\pm0.008$ & $0.675\pm0.021$ & $0.902\pm0.015$ & $0.608\pm0.026$ \\
Qwen-3B & 5.8 & $0.120\pm0.019$ & $0.187\pm0.004$ & $0.700\pm0.014$ & $0.897\pm0.017$ & $0.481\pm0.044$ \\
\midrule
DeepSeek-6.7B & 1.0 & $0.399\pm0.043$ & $0.351\pm0.019$ & $0.838\pm0.022$ & n/a & $0.861\pm0.015$ \\
DeepSeek-6.7B & 1.9 & $0.274\pm0.012$ & $0.383\pm0.005$ & $0.860\pm0.004$ & n/a & $0.891\pm0.030$ \\
DeepSeek-6.7B & 2.9$^{\dagger}$ & $0.289\pm0.046$ & $0.394\pm0.007$ & $0.857\pm0.013$ & n/a & $0.889^{\dagger}$ \\
DeepSeek-6.7B & 3.8 & $0.285\pm0.030$ & $0.413\pm0.013$ & $0.852\pm0.014$ & n/a & $0.881\pm0.026$ \\
DeepSeek-6.7B & 5.8 & $0.185\pm0.022$ & $0.413\pm0.009$ & $0.854\pm0.003$ & n/a & $0.888\pm0.024$ \\
\midrule
DeepSeek-6.7B$^{\dagger}$ & 7.7 & 0.274 & 0.428 & 0.851 & n/a & n/a \\
DeepSeek-6.7B$^{\dagger}$ & 8.6 & 0.193 & 0.423 & 0.843 & n/a & n/a \\
DeepSeek-6.7B$^{\dagger}$ & 9.6 & 0.196 & 0.412 & 0.841 & n/a & n/a \\
\bottomrule
\end{tabular}
\end{table*}

\begin{figure}[H]
  \centering
  \includegraphics[width=\textwidth]{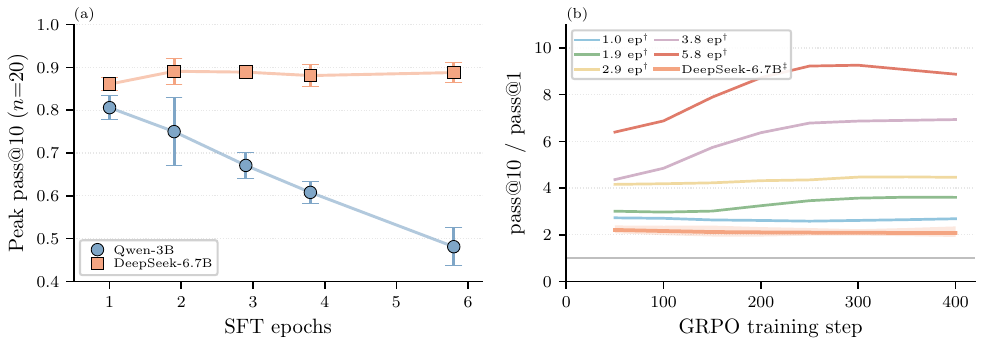}
  \caption{\textbf{Peak outcome and diversity collapse across both
    models.} (a)~Peak pass@10 by SFT depth (3 seed mean, min/max
    bars): Qwen-3B declines monotonically with SFT depth; DeepSeek-6.7B
    stays flat. (b)~Diversity ratio pass@10/pass@1 during GRPO:
    Qwen$^{\dagger}$ reaches $5.7$ to $8.8\times$ for deeper checkpoints
    (per checkpoint lines), with the 1.0 epoch curve staying
    $<2.7\times$; DeepSeek-6.7B$^{\ddagger}$ stays
    $\approx 2.1$ to $2.4\times$ across all GRPO arms (mean line,
    min/max band). $^{\dagger}$Qwen-3B; $^{\ddagger}$DeepSeek-6.7B.}
  \label{fig:peak_by_depth}
\end{figure}

\subsection*{A.6. Pre RL Predictor Correlations (Qwen Ladder)}
\label{app:predictor_corr}

Table~\ref{tab:correlations} reports Spearman rank correlations between
each pre RL predictor and peak GRPO pass@10 (in training $n{=}20$) across
all $n{=}15$ observations (5 checkpoints $\times$ 3 seeds).
Two analyses are shown: the raw pooled correlation and a seed demeaned
version (each observation centred by its seed mean) that isolates the
within seed checkpoint effect.  The significance threshold at $n{=}15$
is $|\rho|>0.521$ ($p{<}0.05$, two tailed). DeepSeek is excluded from
this analysis: pre RL pass@1 spans only $0.351$ to $0.413$ across the
four GRPO arms, leaving no within ladder variance to correlate against
peak pass@10.

\begin{table}[H]
\centering
\caption{Spearman $\rho$ vs.\ peak GRPO pass@10 ($n{=}15$ obs.; pre RL at
$T{=}1.0$, $n{=}128$). $\sigma^2_G = \hat{p}(1-\hat{p})(g-1)/g$.
Stage~2 drop $= (H(10){-}H(150))/H(10)$. Bold: $p{<}0.05$
($|\rho|{>}0.521$).}
\label{tab:correlations}
\small
\begin{tabular}{lrr}
\toprule
Predictor & Pooled & Seed demeaned \\
\midrule
\multicolumn{3}{l}{\textit{Pre RL (Stage 1)}} \\
Pre RL entropy      & $\mathbf{+0.63}$ & $\mathbf{+0.69}$ \\
Pre RL pass@1       & $\mathbf{-0.74}$ & $\mathbf{-0.75}$ \\
Pre RL pass@10      & $-0.47$          & $\mathbf{-0.54}$ \\
Pre RL pass@64      & $\mathbf{+0.60}$ & $\mathbf{+0.57}$ \\
$\sigma^2_G$        & $\mathbf{-0.74}$ & $\mathbf{-0.75}$ \\
\midrule
\multicolumn{3}{l}{\textit{Early GRPO (Stage 2)}} \\
Step 10 entropy $H(10)$   & $\mathbf{+0.76}$ & $\mathbf{+0.82}$ \\
Step 150 entropy $H(150)$ & $\mathbf{+0.75}$ & $\mathbf{+0.75}$ \\
Stage 2 drop              & $\mathbf{-0.67}$ & $\mathbf{-0.69}$ \\
\bottomrule
\end{tabular}
\end{table}

Five findings: (i)~\textbf{pre RL pass@1 is a strong \emph{inverse}
predictor} of GRPO outcome ($\rho{=}{-}0.75$, $p{<}0.01$): choosing the
checkpoint with the highest post SFT pass@1 selects the worst GRPO
performer. The negative sign on $\sigma^2_G$ reflects the same
effect: in the low $p$ regime ($0.151$ to $0.187$), $p(1{-}p)$ is
increasing in $p$, so higher pre RL $\sigma^2_G$ indexes deeper SFT
and predicts worse GRPO collapse, not better gradient signal;
(ii)~\textbf{pre RL entropy is a strong \emph{positive}
predictor} ($\rho{=}{+}0.69$, $p{<}0.01$ seed demeaned), matching the
theoretical mechanism (Proposition~\ref{prop:variance}); (iii)~pre RL
pass@64 is positively correlated ($\rho{=}{+}0.57$, $p{<}0.05$) but is
nearly flat across checkpoints ($0.897$ to $0.932$), limiting its practical
discriminating power; (iv)~pre RL pass@10 is inversely correlated
($\rho{=}{-}0.54$, $p{<}0.05$), reinforcing that standard pass@$k$
metrics at low $k$ select against GRPO trainability; and
(v)~\textbf{early GRPO entropy is an even stronger predictor than pre RL
entropy}: step 10 entropy $H(10)$ reaches $\rho{=}{+}0.82$
($p{<}0.001$), step 150 entropy $\rho{=}{+}0.75$, and the Stage~2
relative drop $\rho{=}{-}0.69$ ($p{<}0.01$).  The correlations use the
in training $n{=}20$ peak pass@10 as the outcome, the quantity the GRPO
training loop actually optimises.  The Stage~2 results in particular
validate the early monitoring diagnostic across all three seeds.

\end{document}